\documentclass[conference]{IEEEtran}
\IEEEoverridecommandlockouts
\usepackage{cite}
\usepackage{amsmath,amssymb,amsfonts}
\usepackage{algorithmic}
\usepackage{graphicx}
\usepackage{textcomp}
\usepackage{xcolor}
\usepackage{times}
\usepackage{epsfig}
\usepackage{graphicx}
\usepackage{amsmath}
\usepackage{amssymb}
\usepackage{subfiles}
\usepackage{cancel}
\usepackage{subfigure}
\usepackage{amsthm}
\usepackage{multirow}
\usepackage{booktabs}
\usepackage{scalerel}
\usepackage{color, colortbl}
\usepackage{threeparttable}
\definecolor{tb_bg_color}{rgb}{0.835, 0.902, 0.941}

\newtheorem{theorem}{Theorem}

\usepackage[pagebackref=true,breaklinks=true,letterpaper=true,colorlinks,bookmarks=false]{hyperref}

\def\BibTeX{{\rm B\kern-.05em{\sc i\kern-.025em b}\kern-.08em
    T\kern-.1667em\lower.7ex\hbox{E}\kern-.125emX}}

\makeatletter 
\newcommand{\linebreakand}{%
  \end{@IEEEauthorhalign}
  \hfill\mbox{}\par
  \mbox{}\hfill\begin{@IEEEauthorhalign}
}
\makeatother 

\begin{document}

\title{WaveNets: Wavelet Channel Attention Networks\\
}

\author{\IEEEauthorblockN{1\textsuperscript{st} Hadi Salman}
\IEEEauthorblockA{\textit{Department of Computer Science } \\
\textit{ University of Arkansas}\\
Fayetteville, Arkansas, USA \\
hs028@uark.edu}
\and
\IEEEauthorblockN{2\textsuperscript{nd} Caleb Parks}
\IEEEauthorblockA{\textit{Department of Computer Science} \\
\textit{ University of Arkansas}\\
Fayetteville, Arkansas, USA \\
cgparks@uark.edu}
\and
\IEEEauthorblockN{3\textsuperscript{rd} Shi Yin Hong}
\IEEEauthorblockA{\textit{Department of Computer Science} \\
\textit{ University of Arkansas}\\
Fayetteville, Arkansas, USA \\
syhong@uark.edu}
\and
\linebreakand 
\IEEEauthorblockN{4\textsuperscript{th} Justin Zhan}
\IEEEauthorblockA{\textit{Department of Computer Science} \\
\textit{University of Cincinnati}\\
Cincinnati, Ohio, USA \\
zhanjt@ucmail.uc.edu}
}
\IEEEoverridecommandlockouts

\IEEEpubid{\makebox[\columnwidth]{978-1-5386-5541-2/18/\$31.00~\copyright2022 IEEE \hfill} \hspace{\columnsep}\makebox[\columnwidth]{ }}
\maketitle

\begin{abstract}
Channel Attention reigns supreme as an effective technique in the field of computer vision. However, the proposed channel attention by SENet suffers from information loss in feature learning caused by the use of Global Average Pooling (GAP) to represent channels as scalars. Thus, designing effective channel attention mechanisms requires finding a solution to enhance features preservation in modeling channel inter-dependencies. In this work, we utilize Wavelet transform compression as a solution to the channel representation problem. We first test wavelet transform as a standalone channel compression method. We prove that global average pooling is equivalent to the recursive approximate Haar wavelet transform. With this proof, we generalize channel attention using Wavelet compression and name it WaveNet. Implementation of our method can be embedded within existing channel attention methods with a couple of lines of code. We test our proposed method using ImageNet dataset for image classification task. Our method outperforms the baseline SENet-34, and SOTA FcaNet-34.

\end{abstract}


\section{Introduction}
\label{sec_intro}
\color{black}
In deep convolutional neural networks (CNNs), effective feature learning often relies upon the success of attention mechanisms in selectively capturing and preserving relevant important details from input \cite{NIU202148}. In tasks such as image classification, attention mechanisms involve redistributing the weights of input feature maps to achieve better classification accuracy \cite{9384588,SMMUF}. Major attention mechanisms used in CNNs consist of channel attention, spatial attention, branch attention, and temporal attention. \cite{arXiv-survey,SSLFFMVDC}. Particularly, the computer vision domain conventionally adopts the channel attention (CA). Introduced by the SeNet \cite{hu2018squeeze}, CA offers a relatively computationally efficient selection of important channels by generating scalar channel weights, whereby channel-wise computations are performed on features derived from the global average pooling (GAP).

While channel attention is an intuitive technique in capturing salient properties of images, recent studies suggest that CA’s use of global average pooling (GAP) in its architecture hinders its performance. GAP is insufficient in retaining sophisticated details and fails to comply with some task-specific model practices \cite{9157483}. Moreover, GAP's straightforward dimensionality reduction further limits CA’s inter-channel dependencies modeling \cite{9156697}. Our motivation to design WaveNet stems from this need to reassess CA to capture finer details in feature learning. This reassessment should allow CA to redistribute the weights of input feature maps to improve classification accuracy while maintaining CA's computational efficiency.

To address the above limitations, we propose to enhance the feature preservation during downsampling via the discrete wavelet transform (DWT). As a tool in digital signal processing, DWT has various image processing applications in tasks such as image compression, dehazing, classification, denoising, restoration, and watermarking \cite{8575265} \cite{9508165} \cite{9523134}  \cite{Liu_2018_CVPR_Workshops} \cite{9232700}. Essentially, DWT performs pyramidal image decomposition by transforming an image into four sub-bands composed of a lowpass (LL) filter and a bandpass filter with horizontal (LH), vertical (HL), and diagonal (HH) decomposition of the image, respectively \cite{shensa1992discrete} \cite{Othman_Zeebaree_2020}. The LL filter corresponds to a downsized version of the original image with lower resolution, and the LH, HL, and HH bandpass filters highlights the input's predominant traits in their associated orientation. The ability of DWT to perform multilevel decomposition on images inspired WaveNet, in which we explore the levels of decomposition of DWT's application in CA. 

In this work, we introduce a novel channel attention framework that stems from a mathematical compression technique. In an effort to better represent channel information and express what GAP failed to explore, we propose to utilize Haar DWT for the channel attention mechanism. 
Along with the Haar channel attention framework, we propose a customized wavelet channel attention framework. In this framework, we use a set of random orthogonal filters to be used in a customized wavelet. The role of those orthogonal filters is to enforce feature preservation and diversity in the compression task prior to excitation of channel attention.

Our implementation of this enhanced channel attention mechanism achieves the state-of-the-art performance against related channel attention techniques. The main contribution of this work are summarized as follows:
\begin{itemize}
	\item We view the channel attention from a compression perspective and adopt DWT in the vanilla channel attention for channel information preservation. With the proof, we establish that conventional GAP is the recurrent Approximate Component of Discrete Haar Wavelet Transform. Then, we generalize the channel attention from the frequency basis and propose our method, termed as WaveNet.
	
	\item Motivated by the success of the Discrete Haar Wavelet Transform in WaveNet, we propose WaveNet-C, a custom orthogonal linearly independent filters wavelet to enforce diversity in the compression task for Channel Attention.  
	
	\item We propose a filter selection criteria along with a parameter reduction technique to fulfill WaveNet-C.
	
	\item We conduct extensive experimental results which support that the presented method achieves the state-of-the-art results on ImageNet comparable computational cost to SENet. 
\end{itemize}

\section{Related Work} 

\color{black}
\paragraph{Visual Attention in CNNs}The active field of research in attention mechanisms has varieties of vision applications across various domains \cite{TIAN2020117} \cite{8576656} \cite{Li_2019_CVPR} \cite{Wu_Huang_Guo_Wang_2020} \cite{9412543}  \cite{9411967} \cite{9619948} \cite{9413102} \cite{9350209} \cite{9018370} \cite{9366353} \cite{9412042}. Early interest in visual attention is fostered by the highway network \cite{srivastava2015highway}, which introduced a gating mechanism that enhances the flow of feature information in a deep neural network. ResNet’s \cite{he2016deep} success with deep CNNs via the use of skips connections in residual blocks further set the foundation for using attention in creating the next state-of-art model. Soon, the proposal of SENet \cite{hu2018squeeze} presents the channel attention in an efficient squeeze and excitation architecture, fueling a wave of studies aiming to improve the channel attention performance. Notably, DANet \cite{fu2019dual} integrates a position attention module with channel attention to model long-range contextual dependencies. Building upon NLNet \cite{wang2018non} and SENet, GCNet \cite{Cao_2019_ICCV} proposes GC blocks to capture channel-wise interdependencies while emphasizing long-range global context modeling. \cite{Peng_2020_CVPR_Workshops} introduces the triplet attention, modeling spatial attention and channel attention with efficient parameters and no dimensionality reduction. HA-CNN \cite{li2018harmonious} assesses joint attention selection, which combines hard regional and soft spatial attention with channel attention. Besides utilizing spatial attention, CBAM \cite{woo2018cbam} applies global max-pooling in channel attention to counter GAP's limits. ECA-Net \cite{wang2020eca} remodels the channel attention architecture to capture cross-channel interaction without unnecessary dimensionality reduction. TSE \cite{Vosco_2021_ICCV} disregards GAP’s global spatial context in SENet to streamline the SE network usage with AI accelerators. FcaNet \cite{Qin_2021_ICCV} adds a multi-spectral component to channel attention from a frequency analysis perspective that explains the relationship between GAP and discrete cosine transform.

\paragraph{Wavelet Transforms in Image Processing} 
Wavelet transforms attract growing interest in deep learning-based image processing applications. Early associated works tend to neglect the use of attention mechanisms and range from image super-resolution \cite{8600724,GHMWTDISR}, classification \cite{9156335} \cite{10.1117/12.2556535}, inpainting \cite{yu2021wavefill}, demoiréing \cite{liu2020wavelet}, and restoration \cite{Liu_2018_CVPR_Workshops}. Recently, some works propose to integrate attention mechanisms and Wavelet transforms. AWNet \cite{Dai2020AWNetAW} integrates non-local attention with DWT to achieve better results for image signal processor (ISP) with smartphone images. \cite{9412623} proposes WAEN, which composes of an attention embedding network and a wavelet embedding network to enhance video super-resolution. The soft attention-based model proposed in \cite{9484398} applies DWT to improve face recognition of morphed images. \cite{8795569} presents a framework for detecting surface defects of glass bottles that fuses the Wavelet transform into their visual attention model. \cite{9190720} details a single image deraining framework based on a fusion network with DWT and its inverse into its attention module. Different than most previous works that use wavelet transform with an attention mechanism for a specific domain-based application, our WaveNet seeks to incorporate Wavelet transform into the underlying architecture of CA to improve the attention mechanism at its most fundamental level.

\newcommand*{\horzbar}{\rule[.5ex]{2.5ex}{0.5pt}}
\newcommand*{\vertbar}{\rule[-1ex]{0.5pt}{2.5ex}}

\section{Method}
\color{black}
We start this section by formulating the Discrete Wavelet Transform (DWT) and Channel Attention (CA). We then look into more details over the derivation of our Interdependent channel attention. Together with the interdependent channel attention model, we explore a diversification strategy for custom wavelet transform.

\subsection{Discrete Wavelet Transform (DWT) and Channel Attention (CA)}
In this section, we first go in-depth over the mathematical derivation of DWT. Then, we elaborate on explaining the channel attention mechanism.

\label{sec_DWTCA}

\paragraph{DWT using Multiplication}

Given scale weights $H$ and shift weights $G$ describing wavelet $w$, the wavelet transform of (1D) input $X$ is:

\begin{equation}
    X_{J=1}^{output} = 
           W \times X = \begin{bmatrix}
           H  \\
           \horzbar \\
           G 
         \end{bmatrix} \times X.
    \label{1d-wav}
\end{equation}

Where \begin{equation}
    W = \begin{bmatrix}
h       &0 &\ldots  &0\\
0 & h      &\ddots  &0\\
\vdots  &\ddots  &h       &0\\
0 &\ldots  &0 &h\\
g       &0 &\ldots  &0\\
0 & g      &\ddots  &0\\
\vdots  &\ddots  &g       &0\\
0 &\ldots  &0 &g
\end{bmatrix}_{n \times n}
    \label{1d-wav1}
\end{equation}

The (2D) DWT can be described using the (1D) DWT by applying the procedure to columns first then repeating the process to the rows of the output. The first level DWT for input X can be modeled as follows:

\begin{equation}
\begin{aligned}
    X_{J=1}^{output} &= DWT(X) =  W \times X \times W^T \\&= \begin{bmatrix} H  \\\horzbar \\G\end{bmatrix} X \begin{bmatrix} H^T  &\vertbar &G^T\end{bmatrix}
    \end{aligned}
    \label{Wave_mult}
\end{equation}

\begin{equation}
    X_{output} =\begin{bmatrix}\mathcal{A} & \mathcal{V}\\\mathcal{H} & \mathcal{D} \end{bmatrix} =\begin{bmatrix} HXH^T & \vertbar & HXG^T \\ \horzbar&&\horzbar\\GXH^T & \vertbar & GXG^T \end{bmatrix}
\end{equation}

Where $\mathcal{A}$ is the Approximation of $X$, $\mathcal{V}$ is the Vertical difference of $X$, $\mathcal{H}$ is the Horizontal difference of $X$, and $\mathcal{D}$ is the Diagonal difference of $X$. For an image with the size of ($n \times n$), the extracted features size is ($n/2 \times n/2$). The wavelet transform level is the number of the times of wavelet feature extraction. At $J^{th}$ level, the extracted feature size is ($n/2^{J} \times n/2^{J}$). 

\paragraph{DWT using Convolution}
\label{DWT_CONV}
Given decoding high pass and low pass filters $H, L$ respectively, we use convolution with correlations of the form :
\begin{equation}
\centering
\label{DWT_CONV_EQ}
Y_{k,l} = \sum \psi_{ij} X_{i+k,j+l}. 
\end{equation}
We assemble the encoding filters by stacking the Low-Low, horizontal, vertical, diagonal filters. For a $d$ value filter where $d \in \{2,3,4,5\}$, the filter bank per channel has the dimensions of size ($4,d,d$). Haar has $d = 2$ value filter.   
For input of size ($N \times C \times H \times W$) the filter bank size for convolution is ($4 \times C, C, d, d$) with 2 as the stride and no padding. The convolution output is of size ($N,C,4,H/2,W/2$) where $\mathcal{A} , \mathcal{V}, \mathcal{H} , \mathcal{D}$ are stacked on $3^{rd}$ dimension.

\paragraph{Channel Attention}

Convolution Neural Networks rely heavily on channel attention mechanisms. The idea is to re-calibrate the channel weights based on relative importance to the general task. Suppose that $X \in \mathbb{R}^{C \times H \times W}$ is an instance of a deep image feature, $C$ is the channel count, $H$ is the feature height, and $W$ is the feature width. As discussed in Sec. \ref{sec_intro}, the channel attention process aims to summarize the channel content into a scalar value. Hence the channel attention mechanism described initially by SENet \cite{} can be written as:
\begin{equation}
\begin{aligned}
    att &= excite(squeeze(X))) \\&= sigmoid(fc(GAP(X)))
\end{aligned}
\label{eq_ca}
\end{equation}
where $att \in \mathbb{R} ^ {C}$ is attention vector, $sigmoid$ is Sigmoid function, $fc$ is a mapping function such as a fully connected layer or an one-dimensional convolution, and squeeze (GAP): $\mathbb{R}^{C \times H \times W} \mapsto \mathbb{R}^{C\times 1 \times 1} $ is a compression method. After acquiring the attention vector of all $C$ channels, all channels of input $X$ are scaled by their corresponding importance value:

\begin{equation}
    \widetilde{X}_{N,C,W,H} = att_{N,C,1,1} \cdot  X_{N,C,W,H}
\end{equation}
where $att$, $X$, and $\widetilde{X}$ are the input, attention vector and attention mechanism output.

Typically, global average pooling is used as the compression method \cite{hu2018squeeze,wang2020eca}. Other popular compression methods are global max pooling~\cite{woo2018cbam} and global standard deviation pooling~\cite{lee2019srm}. 



\subsection{Interdependent Channel Attention}
\label{sec_ica}
In this section, we start by highlighting weaknesses of the current channel attention mechanisms. Based on the theoretical analysis, we then discuss our proposed design to overcome those weaknesses.

\paragraph{Wavelet Channel Attention}

As discussed in Sec.~\ref{sec_DWTCA}, DWT extracts four main features of an image. With the proof, We demonstrate that GAP is equivalent to the recurrent approximation of the input image when Haar wavelet transform is used. 
\begin{figure*}[t]
    \centering
    \subfigure[Original SENet]{\includegraphics[width=0.93\textwidth]{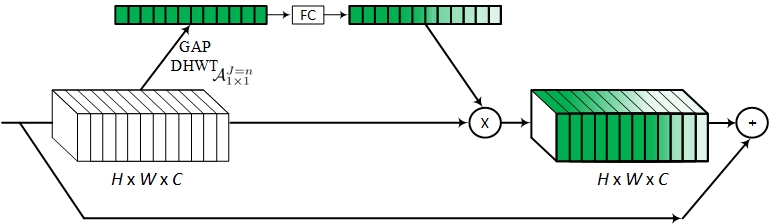}}
    \subfigure[WaveNet-C Orthogonal Interdependent Channel Attention]{\includegraphics[width=0.93\textwidth]{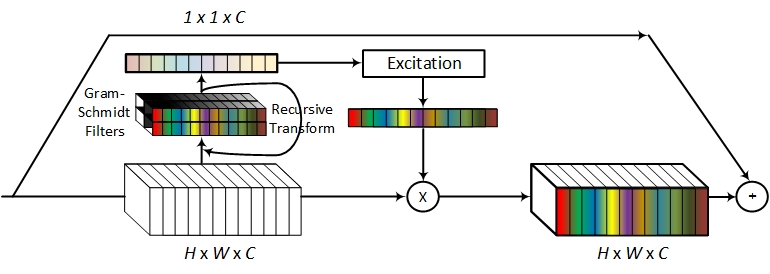}}
    \caption{Illustration of existing channel attention and Orthogonal Interdependent channel attention. The 2D DWT are initialized randomly then Orthogonalized using Gram-Schmidt process. We can see that our method uses a variety of filters, while SENet only uses GAP in channel attention. Best viewed in color.}
    \label{fig_main}
\end{figure*}

\begin{theorem}
    For an image X with the size of $H \times W$, GAP is an exceptional case of 2D DWT with result proportional to the $\log_2{(\max{(H,W)})}$ level approximation using 2D Discrete Haar Wavelet Transform (DHWT). 
    \label{theo}
\end{theorem}

\begin{proof} 
The proof is divided into two transforms. The first transform is applied to the input image X to get a padded version with equivalent global average pooling. If the image $X$ isn't divisible by 2 in both dimensions, we pad X to get an image $A = P(X)$ with $GAP(X) = GAP(A)$. If the input image is already divisible by 2 in both dimensions, we define $P$ to be the identity function. 

Next we get $GAP(B) = GAP(P(A)) = GAP(X)$. You can repeat this argument until $B$ is a 1x1 and $GAP(B) = b$. To do so, we introduce the second transform T. If $B = T(A)$, for $T$ being the transform, we have 
\begin{equation}
    \centering
B_{i, j} = (A_{2i, 2j} + A_{2i+1, 2j} + A_{2i, 2j + 1} + A_{2i+1, 2j+1})/4. 
\end{equation}
From this, it follows that 
\begin{equation}
    \centering
    \sum_{i=1, j=1}^{I, J} B_{i, j} = 1/4 * \sum_{i=1, j=1}^{2*I, 2*J} A_{i, j}
\end{equation}

which implies that $GAP(B) = GAP(A)$. Since $GAP(A) = a$ if $a$ is a 1x1 matrix, the proof is complete by induction.
\end{proof}

\paragraph{Orthogonal Linearly Independent Channel Attention Module}

Theoretical analysis and Theorem \ref{theo} support that
GAP in the channel attention mechanism only uses the average approximation feature
while a diverse variety of potential features are discarded. However, the discarded features may also encode the useful information patterns in representing the channels and should be taken into consideration
in the compression phase.
To mathematically derive a more diverse and meaningful compression method of channel information, we propose to generalize GAP to more wavelet filters and compress more information with multiple different wavelets. 

ResNet has two main blocks, Basic Block and Bottleneck Block. Basic has 4 channel sizes 64, 128, 256, and 512. Bottleneck has 6 channel sizes 64, 128, 256, 512, 1024, and 2048. In case of basic Block, we initialize 4 random interdependent orthogonal filters of same size as channels and we train the network on those filters. For the Bottleneck, we use the same filters for the channels sizes that are shared with the Basic Block. For the channels 1024 and 2048, we split the channels to chunks of size 512 and we initialize those extra 4 filters of size 512 with new random orthogonal weights to enforce catching more diverse information during the compression phase.    

The input $X$  is passed through a separate orthogonal linearly independent wavelet compression module to represent diverse interdependent channel information. In this way, we express the basic compression ($C_B$) as follows:
\vspace{-2pt}
\begin{equation}
    \begin{aligned}
        C_{B}(X) & = \text{DWT}_J(X), \\ 
    \end{aligned}
    \label{Compress_Basic}
\end{equation}
in which the recursive wavelet level J = $\log_2\{H\}$. $X \in \mathbb{R}^{C \times H \times W}$ is the input feature, and $C(X) \in \mathbb{R}^{C}$ is the $C$-dimensional vector post compression. Similarly, bottleneck  compression ($C_{BN}$) is described as follows:

\vspace{-2pt}
\begin{equation}
    \label{Compress_Bottle}
    \centering
C(X)_{BN} =
\begin{cases}
C_{B}(X) &C = 64, 128, 256, 512 \\
CAT(C_{B}(X_{512}))  &C = 1024, 2048
\end{cases}        
\end{equation}

where $CAT$ is the concatenation function along the channel dimension and $X_{512}$ is the split of the input $X$ of size 512 along the channel dimension.

The final orthogonal interdependent channel attention  can be expressed as:
\begin{equation}
    Attention(X) = sigmoid(fc(C(X))).
    \label{eq_final}
\end{equation}

From Eqs.~\ref{Compress_Basic}, \ref{Compress_Bottle} and \ref{eq_final}, it is demonistrated that out model performs a set of Wavelet transforms and extracts channel diverse compression representations of channel information. By incorporating those extra information in the final description we notice a major improvement in the channel representation. Fig.~\ref{fig_main} illustrates the overall concept of our method.


\begin{table*}[htbp]
	\centering
	\caption{Results of the image the classification task on ImageNet over different methods. Besides the AANet, which had no official code implementattion, all methods' results are reproduced and trained with the same training setting. }
	\label{classification}
	\begin{threeparttable}                                        
\begin{tabular}{lcccccccc}
	\toprule
	Method & Years & Backbone & Parameters & FLOPS & Train FPS & Test FPS & Top-1 acc & Top-5 acc \\
	\hline
	ResNet~\cite{he2016deep} & CVPR16 & \multirow{8}{*}{ResNet-34} & 21.80 M & 3.68 G & 2898 & 3840 & 74.58 & 92.05 \\
	SENet~\cite{hu2018squeeze} & CVPR18 &  & 21.95 M & 3.68 G & 2729 & 3489 & 74.83 & 92.23 \\
	ECANet~\cite{wang2020eca} & CVPR20  &  & 21.80 M & 3.68 G & 2703 & 3682 &74.65&92.21\\ 
	FcaNet-LF  & ICCV21 &  & 21.95 M  & 3.68 G & 2717 & 3356 & 74.95 & 92.16\\
	FcaNet-TS  & ICCV21 &  & 21.95 M  & 3.68 G & 2717 & 3356 & 75.02 & 92.07\\
	FcaNet-NAS  & ICCV21 &  & 21.95 M  & 3.68 G & 2717 & 3356 & 74.97 & 92.34\\
	\rowcolor{tb_bg_color}
	WaveNet-C  & BigData22 &  & 21.95 M  & 3.68 G & 2717 & 3356 & \textbf{75.06} & 92.376\\
	\bottomrule
	\end{tabular}
	
\end{threeparttable}
\end{table*}

\paragraph{Wavelet Filter Choice}
One important decision for the network is to pick the wavelet to perform on a specific channel. Our baseline network named Wavenet perform Haar approximation on all channels and achieve SENet results. In order to fulfill the orthogonal interdependent channel attention, we propose Wavenet-C. We discuss more about those networks in the following subsections. 

\textbf{WaveNet} means WaveNet weights the components of wavelet compression within each step of the deep wavelet compression. Its main idea is to improve the compression by including the vertical, horizontal, and diagonal components. First, the network determines the importance of each frequency component. Then, it investigates the effect of adding those frequency components together through the recurrence process.

\textbf{WaveNet-C} means WaveNet with selective wavelet filters. We use the convolution based wavelet transform and we assign orthogonal independent filters for channel compression. We do so by randomly initializing the filters then applying the gram-schmidt process to orthogonality those filters thus forcing the network to diversify the information compressed by each channel therefore achieving better classification in general.

\section{Experiments}
In this section, we began by describing the experimental details of our implementation. Then, we discuss the technique of information compression in our framework, complexity, and code implementation. Lastly, we discuss the accuracy of our method on image classification, object detection, and instance segmentation tasks. 

\subsection{Implementation Details}
\label{Implementation Details}
We utilize  ResNet-34, as backbone model to evaluate the proposed WaveNet on ImageNet~\cite{ILSVRC15}. We comply with data augmentation and hyper-parameter settings in~\cite{he2016deep} and~\cite{he2019bag}. Specifically, with random horizontal flipping, the input images are cropped randomly to 256$\times$256. To do so, we modify ResNet architecture to allow the input size to be 256 instead of 224. During training, the SGD optimizer is set with a momentum of 0.9. The learning rate is 0.2, the weight decay is 1e-4, and the batch size is 256 per GPU. All models are trained within 100 epochs using Cosine Annealing Warm Restarts learning schedule and label smoothing. To foster convergence, for every 10 epochs, the learning rate scales by $10\%$ of the previous learning rate. We further adopt the Nvidia APEX mixed precision training toolkit and Nvidia DALI library for fast data loaders for training efficiency.

All models are implemented in PyTorch~\cite{paszke2019pytorch} and tested on two Nvidia Quadro RTX 8000 GPUs.

\subsection{Discussion}

\paragraph{How the Orthogonal Linearly Independent filters compresses and embeds more information}
In Sec.~\ref{sec_ica}, we prove that solely adopting the vanilla GAP in the channel attention discards information from all filters except the Haar filter, i.e., GAP. Therefore, designing the filters to be orthogonal and linearly independent using the Gram-Schmidt method would force the network to diversify the information extracted in the channel attention compression phase.

We also provide a theoretical basis to show that more information could be embedded. By nature, deep networks are redundant \cite{he2017channel,zhuang2018discrimination}. If two channels contain redundant information, then the application of GAP on these channels are likely to return repetitive information. On the other hand, our multi-spectral framework extracts less superfluous information from redundant channels since the inherent diverging frequency components contain different information. Thus, our multi-spectral framework can embed more unique salient information in the channel attention mechanism.

\paragraph{Complexity analysis}
We analyze the complexity of our framework through the number of parameters and the computational cost. Our method does not impose no extra parameters compared with the baseline SENet that introduced the vanilla channel attention since the filters of 2D DWT are pre-computed constant. The negligible increase in the computational cost is also similar to computational cost of SENet. With ResNet-34 backbone, the relative computational cost increases of our method is $0.05\%$ compared with SENet, respectively. More results can be found in Table~\ref{classification}.

\paragraph{A Few lines of code change}
Another strength of the proposed wavelet attention framework is that it can be integrated into existing diverse variants of channel attention implementations. The major distinction between our method and SENet is the adoption of different channel compression method (multi-spectral 2D DWT vs. GAP). As discussed in Sec.~\ref{DWT_CONV} and Eq.~\ref{DWT_CONV_EQ}, 2D DWT can be viewed as a constant filter convolution of inputs. It can be simply implemented via a Conv2D layer. Accordingly, arbitrary channel attention methods can adopt our framework easily.

\section{Conclusion}

In this paper, we proposed the WaveNet, an efficient, flexible framework for improving channel attention's power in capturing salient features that can easily incorporate into existing channel attention-based models. Theoretically, we prove that the conventional GAP is the recurrent approximation component of the DHWT that discards all channel information in all filters except the Haar filter. Hence WaveNet tackles channel attention as a compression problem and introduces DWT to preserve more unaccounted channel-wise features under GAP. We further introduce WaveNet-C, a custom orthogonal linearly independent wavelet to best fit the compression task for channel attention, and effective wavelet filter selection criteria and parameter reduction techniques. Empirically, our method persistently improves the performance of channel attention mechanism in ImageNet classification task without raising significant parameters and computation costs relative to existing frameworks. Our future works include extending our method for bigger ResNet networks like ResNet-50, ResNet-101; introducing other tasks and datasets like segmentation and object detection on COCO dataset; and incorporating delayed learning for the wavelet filters to further improve our method accuracy.

\section*{Acknowledgment}

This work was supported by the National Institute of General Medical Sciences of the National Institutes of Health under award P20GM139768, and the Arkansas Integrative Metabolic Research Center at the University of Arkansas. The content is solely the responsibility of the authors and does not necessarily represent the official views of the National Institutes of Health.

\bibliographystyle{IEEEtran}
\bibliography{egbib}
\end{document}